\documentclass[letterpaper, 10 pt, conference]{ieeeconf}  % Comment this line out if you need a4paper

\IEEEoverridecommandlockouts                              % This command is only needed if 
\usepackage{graphics} % for pdf, bitmapped graphics files
\usepackage{epsfig} % for postscript graphics files
\usepackage{times} % assumes new font selection scheme installed

\usepackage{amsmath,amsthm,amssymb}
\usepackage{mathtools}
\mathtoolsset{showonlyrefs}
\usepackage{mathrsfs}
\usepackage{cases}
\usepackage{color}
\usepackage{enumerate}
\usepackage{algorithm}
\usepackage{algorithmic}
\usepackage[caption=false]{subfig}
\usepackage{pgfplots}

\newtheorem{theorem}{Theorem}

\newtheorem{corollary}[theorem]{Corollary}
\newtheorem{proposition}[theorem]{Proposition}

\newtheorem{observation}[theorem]{Observation}

\newcommand{\de}{\partial}
\newcommand{\R}{\mathbb{R}}
\newcommand{\st}{\mathrm{s.t.}}
\newcommand{\tr}{^T}
\newcommand{\inv}{^{-1}}
\newcommand{\mc}{\mathcal}

\definecolor{myblue}{RGB}{0,114,89}
\definecolor{myorange}{RGB}{217,83,25}
\definecolor{myyellow}{RGB}{237,177,32}
\definecolor{mypurple}{RGB}{126,47,142}
\definecolor{mygreen}{RGB}{119,172,48}
\definecolor{mycyan}{RGB}{77,190,238}
\definecolor{myred}{RGB}{162,20,47}

\title{\LARGE \bf
Constraint-Driven Coordinated Control of Multi-Robot Systems
}

\author{Gennaro Notomista$^{1}$ and Magnus Egerstedt$^{2}$% <-this % stops a space
\thanks{*This work was sponsored by the U.S. Office of Naval Research through Grant No. N00014-15-2115.}% <-this % stops a space
\thanks{$^1$G. Notomista is with the School of Mechanical Engineering, Institute for Robotics and Intelligent Machines, Georgia Institute of Technology, Atlanta, GA, USA {\tt\small \href{mailto:g.notomista@gatech.edu}{g.notomista@gatech.edu}}}
\thanks{$^2$M. Egerstedt is with the School of Electrical and Computer Engineering, Institute for Robotics and Intelligent Machines, Georgia Institute of Technology, Atlanta, GA, USA {\tt\small \href{mailto:magnus@gatech.edu}{magnus@gatech.edu}}}
}

\begin{document}

\maketitle
\thispagestyle{empty}
\pagestyle{empty}

%%%%%%%%%%%%%%%%%%%%%%%%%%%%%%%%%%%%%%%%%%%%%%%%%%%%%%%%%%%%%%%%%%%%%%%%%%%%%%%%
\begin{abstract}
	
	In this paper we present a reformulation--framed as a constrained optimization problem--of multi-robot tasks which are encoded through a cost function that is to be minimized. The advantages of this approach are multiple. The constraint-based formulation provides a natural way of enabling long-term robot autonomy applications, where resilience and adaptability to changing environmental conditions are essential. Moreover, under certain assumptions on the cost function, the resulting controller is guaranteed to be decentralized. Furthermore, finite-time convergence can be achieved, while using local information only, and therefore preserving the decentralized nature of the algorithm. The developed control framework has been tested on a team of ground mobile robots implementing long-term environmental monitoring.
	
\end{abstract}

%%%%%%%%%%%%%%%%%%%%%%%%%%%%%%%%%%%%%%%%%%%%%%%%%%%%%%%%%%%%%%%%%%%%%%%%%%%%%%%%
\section{INTRODUCTION}
\label{sec:intro}

Robotic swarms are gradually leaving academic laboratories, e.\,g. \cite{pickem2017robotarium,rubenstein2012kilobot}, in favor of industrial settings, as in the case of \cite{d2012guest}, to reach even less structured and more dynamic environments, like agricultural lands and construction sites \cite{parker2008multiple}. This presents new challenges that robots have to face, which come either from unexpected and/or unmodeled phenomena or from changing environmental conditions. These issues become even more pronounced when the robots are deployed on the field for long-term applications, like persistent environment surveiling \cite{stump2011multi} or plant growth monitoring \cite{english2014vision}. Therefore, a way of encoding \textit{survivability} \cite{egerstedt2018robot}, i.\,e., the ability to remain alive (in a robotic sense), is needed now more than ever. In this paper, we introduce a method which can be used to ensure survivability, and which makes use of optimization tools that allow to only minimally influence the task which the robots are asked to perform.

Several solutions have been proposed in order to make robots robust to unknown or changing environmental conditions and to ensure their applicability to unstructured or even hazardous environments \cite{arkin1990techniques}. Moreover, in order to let the robots \textit{survive} for as much time as possible, some of the proposed methods entail scheduled periodic maintenance \cite{mishra2016battery}, or path optimization with the aim of maximizing the time spent in the field/minimizing the consumed energy \cite{martin2014long}. Some other methods employ a power-dependent multi-objective optimization to ensure that the robots execute their task, while maintaining a desired energy reserve, as done in \cite{derenick2011energy}. In both cases, a careful parameters tuning is required in order to prevent situations in which the robots trade survivability for shorter-term rewards.

As a matter of fact, \textit{goal-oriented} control strategies may not be ideal for long-term applications, where robustness to changing environmental conditions is required. Indeed, control policies obtained using optimal control strategies are characterized by a fragility related to the precise model assumptions \cite{csete2002reverse}. These are likely to be violated during the long time horizons over which the robots are deployed in the field. For this reason, in this paper, we consider a \textit{constraint-oriented} approach for multi-robot systems, where the survivability of the swarm is enforced as a constraint on the robots' task, encoded by the nominal input to the robots, $u_{nom}$. The control input $u$ is, then, synthesized, at each point in time, by solving the optimization problem
\begin{align}\label{eq:unom}
\min_u&~\|u-u_{nom}\|^2\\
\mathrm{s.t.}&~c_{surv}(x,u)\ge0,
\end{align}
where $c_{surv}(\cdot,\cdot)$ is the survivability constraint, which is also a function of the robots' state $x$ \cite{egerstedt2018robot}.

It is informative to note that ecological studies have shown that the constraints imposed by an environment strongly determine the behaviors developed by animals living in it \cite{Ricklefs}. Inspired by this concept, we ask whether robotic swarms can be controlled using constraints only. What this entails is that robots are programmed to do nothing, subject to task and survivability constraints. This is formalized by the following optimization problem
\begin{align}\label{eq:nounom}
\min_u&~\|u\|^2\\
\mathrm{s.t.}&~c_{surv}(x,u)\ge0\\
&~c_{task}(x,u)\ge0,
\end{align}
where $c_{task}(\cdot,\cdot)$ encodes the task constraint, which is equivalent to executing the nominal input $u_{nom}$ in \eqref{eq:unom}. Thus, what we could call a \textit{robot-ecological} formulation \cite{egerstedt2018robot} naturally lends itself to be implemented using optimization-based control techniques.

In this paper, we first give sufficient conditions for turning certain classes of multi-robot tasks into constraints within an optimization problem. Then, we present a systematic way of doing this. And, finally, we propose an effective task prioritization technique obtained by combining hard and soft constraints, such as survivability and task execution in \eqref{eq:nounom}.

The remainder of the paper is organized as follows. In Section~\ref{sec:constraintbasedcontrol}, we briefly recall the control techniques that will be used to synthesize the constraint-driven coordinated control policies for multi-robot systems. Then, we introduce an optimization program for executing minimum-energy gradient flow, which will be used, in Section~\ref{sec:multirobot}, to formalize the constraint-based control of multi-robot systems. Moreover, we show how to achieve \textit{decentralized finite-time} minimization algorithms using the presented approach. Two applications are presented in Section~\ref{sec:applications}, namely formation control and coverage control. Section~\ref{sec:experiments} reports the results of experiments executed with a team of ground mobile robots implementing the proposed controller to execute a long-term environmental monitoring task.

\section{CONSTRAINT-BASED CONTROL DESIGN}
\label{sec:constraintbasedcontrol}

In this section, we review the concepts related to control Lyapunov functions and control barrier functions which will then be used to formulate optimization problems whose solution corresponds to the execution of decentralized coordinated controllers for multi-robot systems.

\subsection{Control Lyapunov and Control Barrier Functions}
\label{subsec:clfcbf}

In order to design controllers that allow the execution of multi-robot tasks, we make use of control Lyapunov functions and, in particular, we resort to methods from finite-time stability theory of dynamical systems.

Consider the dynamical system in control affine form
\begin{equation}\label{eq:controlaffine}
	\dot x = f(x) + g(x) u,
\end{equation}
with $x\in \R^n$, $u\in U\subseteq\R^m$, and $f$ and $g$ locally Lipschitz continuous vector fields. One of the results we will use is given by the following theorem.

\begin{theorem}[Based on Theorem~4.3 in \cite{bhat2000finite}]
	\label{thm:finitetimecbf}
	Given a dynamical system \eqref{eq:controlaffine} and a continuous, positive definite function $V:\R^n\to\R$, a continuous controller $u$ such that
	\[
	\inf_{u\in U} \left\{L_f V(x) + L_g V(x) u + c (V(x))^\gamma\right\}\le0\quad\forall x\in\R^n,
	\]
	where $L_fV(x)$ and $L_gV(x)$ denote the Lie derivatives of $V$ in the directions of $f$ and $g$, respectively, $c>0$ and $\gamma\in(0,1)$, renders the origin $x=0$ finite-time stable.
	
	Moreover, an upper bound for the settling time $T$ is given by
	\[
	T \le \frac{1}{c(1-\gamma)} \left(V(x_0)\right)^{1-\gamma},
	\]
	where $x_0$ is the value of $x(t)$ at time $t=0$.
\end{theorem}
\begin{proof}
	Similar to Theorem~4.3 in \cite{bhat2000finite}.
\end{proof}

To enforce constraints, such as survivability or task execution for multi-robot systems, we employ control barrier functions. These, as will be shown, are suitable for synthesizing constraints that can be encoded in terms of set-membership. Conceptually similar to control Lyapunov functions, control barrier functions have been introduced in \cite{ames2014control} with the objective of ensuring safety in a provably correct way. In this context, ensuring safety means ensuring the forward invariance of a set $S\subset \R^n$, in which we want to confine the state $x(t),~\forall t\ge0$. Without loss of generality, here we consider forward complete systems, for which $x(t)$, solution to \eqref{eq:controlaffine}, exists for all $t\ge0$.

Suppose we can find a continuously differentiable function $h:\R^n\to\R$, such that the \textit{safe set} $S$ can be defined as the zero-superlevel set of $h$, i.\,e.,
\begin{equation}\label{eq:safeset}
S = \{ x\in \R^n~|~h(x)\ge0 \},
\end{equation}
and
\begin{align}
\partial S &= \{ x\in \R^n~|~h(x)=0 \}\\
S^\circ &= \{ x\in \R^n~|~h(x)>0 \},
\end{align}
where $\partial S$ and $S^\circ$ denote the boundary and the interior of $S$, respectively.
If the following condition is satisfied
\begin{equation}\label{eq:zcbf}
\sup_{u\in U} \left\{L_f h(x) + L_g h(x) u + \alpha(h(x))\right\} \ge 0\quad\forall x\in \R^n,
\end{equation}
with $\alpha$ an extended class $\mc K$ function \cite{kellett2014compendium}, then $h$ is called a \textit{(zeroing) control barrier function}. Conditions to ensure forward invariance of the set $S$ are given in the following theorem.

\begin{theorem}[Safe set forward invariance \cite{xu2015robustness}]\label{thm:cbf}
Given a dynamical system \eqref{eq:controlaffine} and a set $S\subset\R^n$ defined by a continuously differentiable function $h$ as in \eqref{eq:safeset}, any Lipschitz continuous controller $u$ such that \eqref{eq:zcbf} holds renders the set $S$ forward invariant.
\end{theorem}
\begin{proof}
	See \cite{xu2015robustness}.
\end{proof}

The existence of the control barrier function $h$ also ensures the asymptotic stability of the set $S$, as shown in the following theorem.

\begin{theorem}[Safe set asymptotic stability]
	\label{thm:asymptstab}
	Under the same hypotheses as in Theorem~\ref{thm:cbf}, any Lipschitz continuous controller $u$ such that \eqref{eq:zcbf} holds renders the set $S$ asymptotically stable, that means $x(t)\vert_{t=0}\notin S \Rightarrow x(t)\rightarrow\in S$ as $t\to\infty$.
\end{theorem}
\begin{proof}[Proof (based on \cite{xu2015robustness})]
	Let
	\begin{equation}
	V(x)=\begin{cases}
	-h(x)&x\in\R^n\setminus S\\
	0&x\in S
	\end{cases}
	\end{equation}
	be a control Lyapunov candidate function. Thus, $V(x)>0$ for $x\in\R^n\setminus S$ and $V(x)=0$ for $x\in S$. Moreover,
	\begin{align}
	\dot V &= \frac{\de V}{\de x}\dot x = L_f V(x) + L_g V(x) u\\
	&=\begin{cases}
		-L_f h(x) -L_g h(x) u&x\in\R^n\setminus S\\
		0&x\in S.
	\end{cases}
	\end{align}
	Furthermore, since $h$ is continuously differentiable, $V$ is continuously differentiable as well. Then, by hypothesis \eqref{eq:zcbf}, $\dot V = -L_f h(x) -L_g h(x) u \le \alpha(h(x)) < 0$ for $x\in\R^n\setminus S$ and $\dot V = 0$ for $x\in S$. By Theorem~\ref{thm:cbf}, $S$ is forward invariant. Moreover, $S$ is closed, since it is the inverse image of the closed set $[0,\infty)\subseteq\R$ under the continuous map $h$. Therefore, by Theorem~2.8 in \cite{lin1996smooth}, the system \eqref{eq:controlaffine} is uniformly globally asymptotically stable with respect to the set $S$. Thus, there exists a class $\mc{KL}$ function $\beta$ \cite{kellett2014compendium} such that, given any initial state $x_0$, the solution $x(t)$ satisfies $d(x(t),S)\le\beta(d(x_0,S),t),~\forall t\ge0$, where $d(y,S)\triangleq\inf_{z\in S}\|y-z\|$. Hence, as $t\to\infty$, $x(t)\rightarrow\in S$.
\end{proof}

\subsection{Minimum-Energy Gradient Flow}
\label{subsec:gradientconstraint}

In this section we consider the problem of minimizing a cost function $J$. We present a method to reformulate the classic gradient flow algorithm using the tools introduced in Section~\ref{subsec:clfcbf}. This allows us to synthesize a constrained optimization program that is equivalent to the minimization of the cost $J$.

Consider the single integrator dynamical system $\dot x = u$, where $x,u\in\R^n$ are the state and the control input, respectively. Assume that the objective consists in minimizing a cost $J(x)$, where $J:\R^n\to\R_+$ is a continuously differentiable function. Applying gradient flow algorithms, the problem
\begin{equation}
\min_u~J(x)
%\begin{aligned}
%\min_u~&J(x)\\
%\st~&\begin{cases}
%\dot x = u\\
%\left.x(t)\right\vert_{t=0} = x_0,
%\end{cases}
%\end{aligned}
\label{eq:minuJ}
\end{equation}
can be directly minimized by choosing
\begin{equation}
u = -\frac{\de J}{\de x}\tr(x).
\label{eq:optimalucost}
\end{equation}
In fact, with this choice of input, applying chain rule leads to:
\begin{equation}\label{eq:negativejdot}
\dot J(x) = \frac{dJ}{dt} = \frac{\de J}{\de x} \dot x = \frac{\de J}{\de x} u = -\left\|\frac{\de J}{\de x}\right\|^2 \le 0.
\end{equation}

We now show that the minimization problem~\eqref{eq:minuJ} can be formulated as a minimum-energy problem that achieves the same objective of minimizing the cost $J$.

To this end, let us define the barrier function
\begin{equation}\label{eq:hj}
h(x) = -J(x)
\end{equation}
and its zero-superlevel set, i.\,e. the \textit{safe} set,
\begin{equation}
S = \left\{x~|~h(x)\ge0\right\}= \left\{x~|~J(x)\le0\right\}= \left\{x~|~J(x)=0\right\}.
\end{equation}
By Theorems~\ref{thm:cbf}~and~\ref{thm:asymptstab}, the differential constraint
\begin{equation}
\dot h(x) = \frac{\de h}{\de x} u \ge -\alpha(h(x)),
\end{equation}
where $\alpha$ is an extended class $\mc K$ function, ensures that the set $S$ is forward invariant when $h(x(0))\ge0$ and asymptotically stable when $h(x(0))\le0$, $x(0)$ being the value of the state $x$ at time $t=0$.

\begin{observation}\label{obs:lyapbarr}
Theorem~\ref{thm:asymptstab} shows the existence of the control Lyapunov function
\begin{equation}\label{eq:lyapbarrier}
\begin{aligned}
V(x) &= \begin{cases}
-h(x) & \mathrm{if}~x\notin S\\
0 & \mathrm{if}~x\in S
\end{cases}\\
&= \begin{cases}
J(x) & \mathrm{if}~x\notin S\\
0 & \mathrm{if}~x\in S
\end{cases}\\
&\equiv J(x).
\end{aligned}
\end{equation}
Indeed, from \eqref{eq:negativejdot}, since $J(x)\ge0,~\forall x\in\R^n$, one can see that $J(x)$ is a control Lyapunov function. In fact, if $x$ belongs to $X\subset\R^n$ compact, LaSalle's Invariance Principle ensures that the state will converge to a stationary point of $J(x)$, namely, $x\rightarrow x^\ast$, with $\frac{\de J}{\de x}(x^\ast) = 0$.
\end{observation}

We can now introduce the following optimization problem:
\begin{equation}
\begin{aligned}
\min_{u,\delta}~&\|u\|^2+|\delta|^2\\
\st~&\frac{\de h}{\de x} u \ge -\alpha(h(x))-\delta,
\end{aligned}
\label{eq:minubarrier}
\end{equation}
$\delta\in\R$, which solves the problem in \eqref{eq:minuJ}, as shown in the following proposition.

\begin{proposition}\label{eq:equivalence}
The solution of the optimization problem \eqref{eq:minubarrier}, where $h(x)$ is given by \eqref{eq:hj} and $\alpha$ is an extended class $\mc K$ function, solves \eqref{eq:minuJ}, driving the state $x$ of the dynamical system $\dot x = u$ to a stationary point of the cost $J$.
\label{prop:equivalence}
\end{proposition}
\begin{proof}
The KKT conditions for the problem in \eqref{eq:minubarrier} are
\begin{equation}
\begin{cases}
-\dfrac{\de h}{\de x} u^\ast-\alpha(h(x))-\delta^\ast\le0\\
\lambda^\ast\ge0\\
\lambda^\ast\left(-\dfrac{\de h}{\de x} u^\ast-\alpha(h(x))-\delta^\ast\right)=0\\
\begin{bmatrix}
2u^\ast\\2\delta^\ast
\end{bmatrix}+\lambda^\ast\begin{bmatrix}
-\dfrac{\de h}{\de x}\tr\,\\
-1
\end{bmatrix}=0,
\end{cases}
\label{eq:kkt}
\end{equation}
where $u^\ast,\delta^\ast$ and $\lambda^\ast$ are primal and dual optimal points \cite{boyd2004convex}.
First of all, we note that, if $\lambda^\ast=0$, then $u^\ast=0$ by the fourth equation in \eqref{eq:kkt}. Therefore, from the first equation in \eqref{eq:kkt}, $-\alpha(h(x))\le0$. This is equivalent to $-\alpha(-J(x))\le0$ and, since $J(x)\ge0$, this implies that $J(x)=0$. In case $\lambda^\ast>0$, from the third and fourth equation in \eqref{eq:kkt}, one has $\lambda^\ast = -2\alpha(h(x))\left(1+\|\frac{\de h}{\de x}\|^2\right)\inv$, and therefore, $u^\ast = -\alpha(h(x))\frac{\de h}{\de x}\tr\left(1+\|\frac{\de h}{\de x}\|^2\right)\inv$.
Since $J(x)\ge0~\forall x\in\R^n$ and $J$ is continuously differentiable, one can show that $J(\bar x)=0\Rightarrow\left.\frac{\de J}{\de x}\right\vert_{x=\bar x} = 0$. Thus, we can unify the two cases, $\lambda^\ast=0$ and $\lambda^\ast>0$, and write the expression of the optimal $u$ as follows:
\begin{equation}\label{eq:optimaluconstraint}
u^\ast = \frac{\alpha(-J(x))\frac{\de J}{\de x}\tr}{1+\|\frac{\de J}{\de x}\|^2}.
\end{equation}
With this expression of the input $u$, the evolution in time of the cost $J$ is given by
\begin{equation}
\dot J = \frac{\de J}{\de x}\dot x = \frac{\de J}{\de x} u^\ast= \frac{\alpha(-J(x))\|\frac{\de J}{\de x}\|^2}{1+\|\frac{\de J}{\de x}\|^2}.
\end{equation}
So,
\begin{equation}
\frac{\de J}{\de x}\neq0\Rightarrow\dot J<0 \qquad\text{and}\qquad \frac{\de J}{\de x}=0\Rightarrow\dot J=0.
\end{equation}
Hence, as $t\to\infty$, $x(t)\rightarrow x^\ast$, such that $\frac{\de J}{\de x}(x^\ast) = 0$.
\end{proof}
\begin{corollary}
Under the same hypotheses as in Proposition~\ref{prop:equivalence} and $J$ such that $\frac{\de J}{\de x} = 0 \Leftrightarrow x=0$, the solution of the optimization program
\begin{equation}\label{eq:nodelta}
\begin{aligned}
\min_u~&\|u\|^2\\
\st~&\frac{\de h}{\de x} u \ge -\alpha(h(x)),
\end{aligned}
\end{equation}
solves the problem in \eqref{eq:minuJ}.
\end{corollary}
\begin{proof}
Proceeding similarly to the proof of Proposition~\ref{prop:equivalence}, the solution to \eqref{eq:nodelta} evaluates to
%\begin{equation}
$u^\ast = \frac{\alpha(-J(x))\frac{\de J}{\de x}\tr}{\|\frac{\de J}{\de x}\|^2}$,
%\end{equation}
and, therefore, $\dot J = \alpha(-J(x))$. Thus, $J\rightarrow0$ as $t\rightarrow\infty$ \cite{khalil1996noninear}. Hence, as $t\to\infty$, $x(t)\rightarrow x^\ast$ such that $\frac{\de J}{\de x}(x^\ast) = 0$, and so $J(x^\ast) = 0$.
\end{proof}

In summary, we saw that the expression for $u$ given in \eqref{eq:optimaluconstraint}
solves the initial optimization problem \eqref{eq:minuJ}, which can be equivalently solved following the gradient flow of the cost $J$ using \eqref{eq:optimalucost}.

We now illustrate that, besides the advantages related to long-term autonomy applications discussed in Section~\ref{sec:intro}, the formulation in \eqref{eq:minubarrier} can be used to design decentralized cost minimization algorithms that are faster than gradient descent. In the optimization literature, there are plenty of methods that can be employed to improve the convergence speed of gradient flow algorithms (see, e.\,g., \cite{boyd2004convex}). Nevertheless, these \textit{second order} methods, such as Newton's method or conjugate gradient, suffer from their centralized nature. Only in some cases, this issue can be partially mitigated by resorting to distributed optimization techniques, such as \cite{wei2013distributed,boyd2011distributed}. The above-mentioned methods are all suitable for minimizing a cost function. However, we insist on having a constrained optimization formulation--where we encode cost minimization as a constraint--because of the flexibility and robustness properties discussed in Section~\ref{sec:intro}, useful for long-term robot autonomy applications.

The following proposition shows that, using the formulation in \eqref{eq:minubarrier}, it is possible to minimize the cost $J$ and to be not just faster than gradient descent, but actually to reach a stationary point in finite time.

\begin{proposition}\label{prop:finitetime}
Given the dynamical system $\dot x = u$ and the objective of minimizing the cost function $J$, the solution of the optimization problem
\begin{equation}
\begin{aligned}
\min_{u,\delta}~&\|u\|^2+|\delta|^2\\
\st~&\frac{\de h}{\de x} u \ge -c(h(x))^\gamma-\delta,
\end{aligned}
\end{equation}
where $h(x)$ is given by \eqref{eq:hj}, $c>0$ and $\gamma\in(0,1)$, will drive the state $x$ to a stationary point of the cost $J$ in finite time.
\end{proposition}
\begin{proof}
Similarly to what has been done in Propositon~\ref{prop:equivalence}, it can be shown that
\begin{equation}
\dot J = \frac{-c(J(x))^\gamma\|\frac{\de J}{\de x}\|^2}{1+\|\frac{\de J}{\de x}\|^2}.
\end{equation}
Thus, by Theorem~\ref{thm:finitetimecbf}, we conclude that
\begin{equation}
\frac{\de J}{\de x}\neq0 \Rightarrow J\rightarrow 0~\text{in finite time},
\end{equation}
and
\begin{equation}
\frac{\de J}{\de x}=0 \Rightarrow \dot J=0.
\end{equation}
Hence, $x\rightarrow x^\ast$, with $\frac{\de J}{\de x}(x^\ast)=0$, in finite time. Indeed, as shown in \cite{li2018formally}, $h(x)$ such that $\dot h\ge -c(h(x))^\gamma$ is a finite-time convergence control barrier function for the system characterized by single integrator dynamics, $\dot x=u$.
\end{proof}

Using the results derived in this section, the next section presents a procedure to synthesize decentralized optimization problems whose solutions result in coordinated control of multi-robot systems.

\section{CONSTRAINT-BASED CONTROL OF MULTI-ROBOT SYSTEMS}
\label{sec:multirobot}

\textit{Local}, \textit{scalable}, \textit{safe} and \textit{emergent} are four essential features that decentralized multi-robot coordinated control algorithms should possess \cite{cortes2017coordinated}. Many algorithms that satisfy these properties have been developed for applications ranging from social behavior mimicking \cite{reynolds1987flocks}, formation assembling \cite{egerstedt2001formation} and area patrolling \cite{cortes2004coverage}.
In \cite{cortes2017coordinated}, the authors analyze the common features among these algorithms and discuss their decentralized implementations in robotic applications. In this section, we apply the results derived in Section~\ref{sec:constraintbasedcontrol} with the aim of obtaining constrained optimization problems equivalent to the decentralized execution of multi-robot tasks.

Consider a collection of $N$ robots, whose position is denoted by $x_i \in\R^d,~i\in\{1,\ldots,N\}$, where $d=2$ for planar robots and $d=3$ in the case of aerial robots. Assume each robot is equipped with an omni-directional range sensor that allows it to measure the relative position of neighboring robots, namely robot $i$ is able to measure $x_j-x_i$, when robot $j$ is within its sensing range. These interactions among the robots are described by a graph $\mc G=(\mc V,\mc E)$, where $\mc V = \{1,\ldots,N\}$ is the set of vertices of the graph, representing the robots, and $\mc E \subseteq \mc V\times\mc V$ is the set of edges between the robots, encoding adjacency relationships. If $(i,j)\in\mc E$, then robot $i$ can measure robot $j$'s position. For the purposes of this paper, we assume that the graph is undirected, namely $(i,j)\in\mc E \Leftrightarrow (j,i)\in\mc E$. In order to obtain decentralized algorithms, we want each robot to act only based on local information, by which we mean the relative positions of its neighbors. By construction, this leads to inherently scalable coordinated control algorithms.

Denoting the ensemble state of the robotic swarm by $x = [x_1\tr,\ldots,x_N\tr]\tr\in\R^{Nd}$, a general expression for the cost that leads to decentralized control laws is given by
\begin{equation}
J(x) = \sum_{i=1}^N \sum_{j\in\mc N_i} J_{ij}(\| x_i-x_j \|),
\label{eq:pairwisecost}
\end{equation}
where $\mc N_i$ is the neighborhood set of robot $i$, and $J_{ij}:\R\to\R$, $J_{ij}(\|x_i-x_j\|) = J_{ji}(\|x_j-x_i\|)$ is a symmetric, pairwise cost between robots $i$ and $j$. We assume that $J_{ij}(x)\ge0,~\forall (i,j) \in\mc E,~\forall x\in\R^n$, so that $J(x)\ge0,~\forall x\in\R^n$. Assuming we can directly control the velocity of robot $i$, $\dot x_i$, we can employ a gradient descent flow policy like \eqref{eq:optimalucost} to minimize $J$, obtaining
\begin{equation}
u_i = -\sum_{j\in\mc N_i} \frac{\de J_{ij}}{\de \|x_i-x_j\|} \frac{x_i-x_j}{\|x_i-x_j\|} = \sum_{j\in\mc N_i} w_{ij} (x_j-x_i).
\end{equation}
This is nothing but a weighted consensus protocol, and it is decentralized insofar as the input $u_i$ only depends on robot $i$'s neighbors. The construction shown in Section~\ref{sec:constraintbasedcontrol} can be then applied to minimize the cost given in \eqref{eq:pairwisecost} by formulating the following minimum-energy problem:
\begin{equation}\label{eq:minuforJ}
\begin{aligned}
\min_{u,\delta}~&\|u\|^2+|\delta|^2\\
\st~&-\frac{\de J}{\de x} u \ge -\alpha(-J(x))-\delta,
\end{aligned}
\end{equation}
where $u = [u_1\tr,\ldots,u_N\tr]\tr\in\R^{Nd}$ is the vector of robots' inputs, and a single integrator dynamics, $\dot x_i = u_i$, is assumed for each robot. Solving \eqref{eq:minuforJ} leads to the accomplishment of the task, by which we mean that a stationary point of the cost $J$ has been reached. As explicitly shown by \eqref{eq:optimaluconstraint} in Proposition~\ref{prop:equivalence}, a minimum-energy formulation, initially introduced in \cite{ames2014control}, allows the robots %to navigate the energy landscape shaped by the cost $J$,
to move towards lower values of the cost $J$ until the task is accomplished ($u\equiv0$).

The following proposition gives the expression of the optimization problems whose solutions lead to a decentralized minimization of the cost $J$ in \eqref{eq:pairwisecost}.

\begin{proposition}[Constraint-driven decentralized task execution]\label{prop:decentralizedtask}
Given the pairwise cost function $J$ defined in \eqref{eq:pairwisecost}, a collection of $N$ robots, characterized by single integrator dynamics, minimizes $J$ in a decentralized fashion, if each robot executes the control input, solution of the following optimization problem:
\begin{equation}\label{eq:optuJ}
\begin{aligned}
\min_{u_i,\delta_i}~&\|u_i\|^2+|\delta_i|^2\\
\st~&-\frac{\de J_i}{\de x_i} u_i \ge -\alpha(-J_i(x))-\delta_i,
\end{aligned}
\end{equation}
where $J_i(x) = \sum_{j\in\mc N_i} J_{ij}(\| x_i-x_j \|)$ and $\alpha$ is an extended class $\mc K$ function, $\alpha : x\in\R \mapsto \alpha(x)\in\R$, superadditive for $x<0$, i.\,e. $\alpha(x_1+x_2)\ge\alpha(x_1)+\alpha(x_2),~\forall x_1,x_2<0$. If $\alpha(x)=cx^\gamma$, $c>0$, $\gamma\in(0,1)$, a stationary point of the cost $J$ is reached in finite time, with the upper bounds on the settling time given in Theorem~\ref{thm:finitetimecbf}. 
\end{proposition}
\begin{proof}
Proposition~\ref{eq:equivalence} ensures that, by imposing the global constraint $-\frac{\de J}{\de x} u \ge -\alpha(-J(x))$, constructed using the whole state vector $x$, the cost $J$ is decreasing towards a stationary point. We want to show that, by imposing only local constraints (i.\,e., such that robot $i$ only needs information about its neighbors), the multi-robot system is able to enforce the global constraint and, hence, to minimize the cost $J$ in a decentralized fashion.

We proceed by starting to sum up the constraints for each robot, obtaining:
\begin{align}
&\sum_{i=1}^N\left( -\frac{\de J_i}{\de x_i}u_i \right) \ge \sum_{i=1}^N (-\alpha(-J_i(x))-\delta_i)\\
\ge&-\alpha\left(-\sum_{i=1}^N J_i(x)\right)-\delta \ge-\alpha\left(-J(x)\right)-\delta,
\end{align}
where we used the superadditivity property of $\alpha$, and we set $\delta = \sum_{i=1}^N\delta_i$. Moreover, since the graph $\mc G$, which encodes the neighboring relations between the robots, is undirected, we have that $\frac{\de J_i}{\de x_i}=\frac{1}{2}\frac{\de J}{\de x_i}$. Thus,
\begin{equation}
\frac{\de J}{\de x}u\ge-2\alpha\left(-J(x)\right)-2\delta=-\alpha^\prime\left(-J(x)\right)-\delta^\prime,
\end{equation}
where $\frac{\de J}{\de x}=\left[\frac{\de J}{\de x_1},\ldots,\frac{\de J}{\de x_N} \right]$, $u=\left[u_1\tr,\ldots,u_N\tr\right]\tr$, and $\alpha^\prime$ an extended class $\mc K$ function. Hence, by Proposition~\ref{prop:equivalence}, $x$ will converge to a stationary point of $J$.

Finally, we note that a class $\mc K$ function $\alpha(x)=cx^\gamma$, defined for $x<0$ is convex, and hence superadditive, for $x<0$. Applying Proposition~\ref{prop:finitetime}, the statement holds.
\end{proof}

The structure of the cost function  $J(x)$, even though quite specific, allows us to encode a rich set of multi-robot tasks, by carefully choosing the weights $w_{ij}$ as a function of the state $x$. The following section shows two variations on the cost function which allow a multi-robot system to perform formation control, i.\,e., assembling particular shapes, and coverage control, consisting in spreading out the robotic swarm in the environment in an optimal way.

\section{APPLICATIONS}
\label{sec:applications}

In this section we recall the expression of the cost $J$ for two specific multi-robot tasks: formation control and coverage control.

\subsection{Formation Control}
\label{subsec:formation}

In formation control applications, the robots are asked to assemble a predefined shape, specified in terms of inter-agent distances. In order to frame this problem as a cost minimization problem, let $J$ be the formation error
\begin{equation}\label{eq:formationcost}
J(x) = \sum_{i=1}^n \sum_{j\in\mc N_i} \frac{1}{2} (\|x_i-x_j\|-d_{ij})^2 = \sum_{i=1}^n J_i(\|x_i-x_j\|),
\end{equation}
where $d_{ij}$ is the desired distance between robots $i$ and $j$. $J$ measures how far the robots are from assembling the desired formation characterized by the relative distances $d_{ij}$. $J=0$ corresponds to the robots forming the desired shape. Note that, as $J_i(x)$ is a sum of squares, $J_i(x) \le J(x),~\forall i$, required as a hypothesis in order for Proposition~\ref{prop:decentralizedtask} to hold.

The gradient the $J_i(x)$ evaluates to
\begin{equation}\label{eq:gradformation}
\frac{\de J_i}{\de x_i} = \sum_{j\in\mc N_i} \frac{\|x_i-x_j\|-d_{ij}}{\|x_i-x_j\|} (x_i-x_j)\tr.
\end{equation}
This can be interpreted as follows: if the distance between robots $i$ and $j$ is smaller than $d_{ij}$, then the weight $w_{ij} = \frac{\|x_i-x_j\|-d_{ij}}{\|x_i-x_j\|}$ is negative, and the robots experience a repelling effect. Conversely, if the two robots are further than $d_{ij}$ apart, the positive weight $w_{ij}$ will attract one towards the other. The special case in which $d_{ij} = 0,~\forall i,j$ corresponds to the well-known consensus problem.

The expression of the gradient in \eqref{eq:gradformation} is decentralized as robot $i$ has to compute relative distances only with respect to its neighboring robots.

\subsection{Coverage Control}
\label{subsec:coverage}

In coverage control, the task given to the robots is that of covering a domain $D$. Given a coverage performance measure, the robots should spread over the domain in an optimal way. As shown in \cite{lloyd1982least}, each robot should be in charge only of a subset of the domain $D$ that, more specifically, is its Voronoi cell, defined as $V_i = \{ p\in D~|~\|p-x_i\|\le\|p-x_j\|~\forall i\neq j \}$.

Let us introduce the measure of how bad a domain is being covered:
\begin{equation}\label{eq:loccost}
J(x) = \sum_{i=1}^{N}\frac{1}{2} \| x_i-G_i(x) \|^2 = \sum_{i=1}^n J_i(x),
\end{equation}
where $G_i$ denotes the centroid of the Voronoi cell $V_i$. This form is just a reformulation of the locational cost originally introduced in \cite{cortes2004coverage}. Taking the derivative of $J_i$ with respect to $x_i$, required in the optimization problem \eqref{eq:optuJ}, one obtains:
\begin{equation}\label{eq:gradientcoverage}
\frac{\de J_i}{\de x_i} = (x_i-G_i(x))\tr\left( I-\frac{\de G_i(x)}{\de x_i} \right),
\end{equation}
where $I$ is the identity matrix. Note that, even if $G_i(x)$ virtually depends on the entire ensemble state, $x$, of the robotic swarm robot $i$, in order to compute it, only requires information from the robots with which it shares part of the boundary of its Voronoi cells.

In the Appendix we show how the formulation presented in this paper also allows an exact decentralized implementation of the coverage control with time-varying density functions introduced in \cite{lee2015multirobot}.

We deployed the optimization-based control algorithms with the expressions of the costs $J$ derived in Sections~\ref{subsec:formation}~and~\ref{subsec:coverage} on a real multi-robot system and, in the next section, we show the experimental results. Moreover, the constraint-driven formulation of Section~\ref{sec:multirobot} is used to achieve long-term environmental monitoring, where the robots are tasked with covering a domain over a time-horizon which is much longer than their battery life, and during which the robots will also have to avoid collisions with obstacles moving around in the domain.

\section{EXPERIMENTS}
\label{sec:experiments}

The coordinated control approach presented in this paper has been tested on the Robotarium \cite{pickem2017robotarium}, a remotely accessible swarm robotics testbed. The Robotarium is populated by small-scale differential-drive robots which can be programmed by uploading code scripts via a web interface.

Throughout the paper, we assumed we can directly control the velocity of the robots, by modeling them using single integrator dynamics. However, a differential-drive robot can be more accurately modeled using unicycle dynamics:
\begin{equation}\label{eq:uni}
\begin{cases}
\dot x = v \cos(\theta)\\
\dot y = v \sin(\theta)\\
\dot \theta = \omega,
\end{cases}
\end{equation}
where $[x, y]\tr$ and $\theta$ are the robot's position and orientation in the plane, respectively, and $v$ and $\omega$ are the linear and angular velocity inputs, respectively. Nevertheless, in \cite{olfati2002near}, it is shown that it is possible to derive a near-identity diffeomorphism that can be used to partially feedback linearize the system \eqref{eq:uni}. This way, the unicycle can be abstracted as a single integrator. This is realized through the invertible map
\begin{equation}
	\begin{bmatrix}
		v\\
		\omega
	\end{bmatrix} = R\tr(\theta) \begin{bmatrix}
		1&0\\
		0&\frac{1}{d}
	\end{bmatrix}\begin{bmatrix}
		\dot x_d\\
		\dot y_d
	\end{bmatrix},
\end{equation}
where $R(\theta)$ is the matrix that rotates vectors in $\R^2$ counterclockwise by an angle $\theta$, and $[\dot x_d, \dot y_d]\tr$ is the velocity in the plane of a point located in front of the unicycle at a distance $d$ from its center.
This method is used to control the robots on the Robotarium, by calculating linear and angular velocities of robot $i$, $v_i$ and $\omega_i$, from the control input 
%$u_i$,
$u_i=[\dot x_{d,i}, \dot y_{d,i}]\tr$,
obtained by solving the optimization problems derived in Section~\ref{sec:multirobot}. 

Regarding the implementation of the optimization program \eqref{eq:optuJ} needed for the execution of the tasks presented in the previous section, the function $\alpha$ has been chosen to be $\alpha(x) = \sqrt[3]{x}$, which is an extended class $\mc K$ function, convex for $x<0$. This implies it is also superadditive for $x<0$, as required by the hypotheses in Proposition~\ref{prop:decentralizedtask}.

\subsection{Formation and Coverage Control}

\begin{figure*}
	\centering
	\hfill\subfloat[][]{\label{subfig:formationa}\includegraphics[trim={7cm 2cm 14cm 4cm}, clip,width=.24\textwidth]{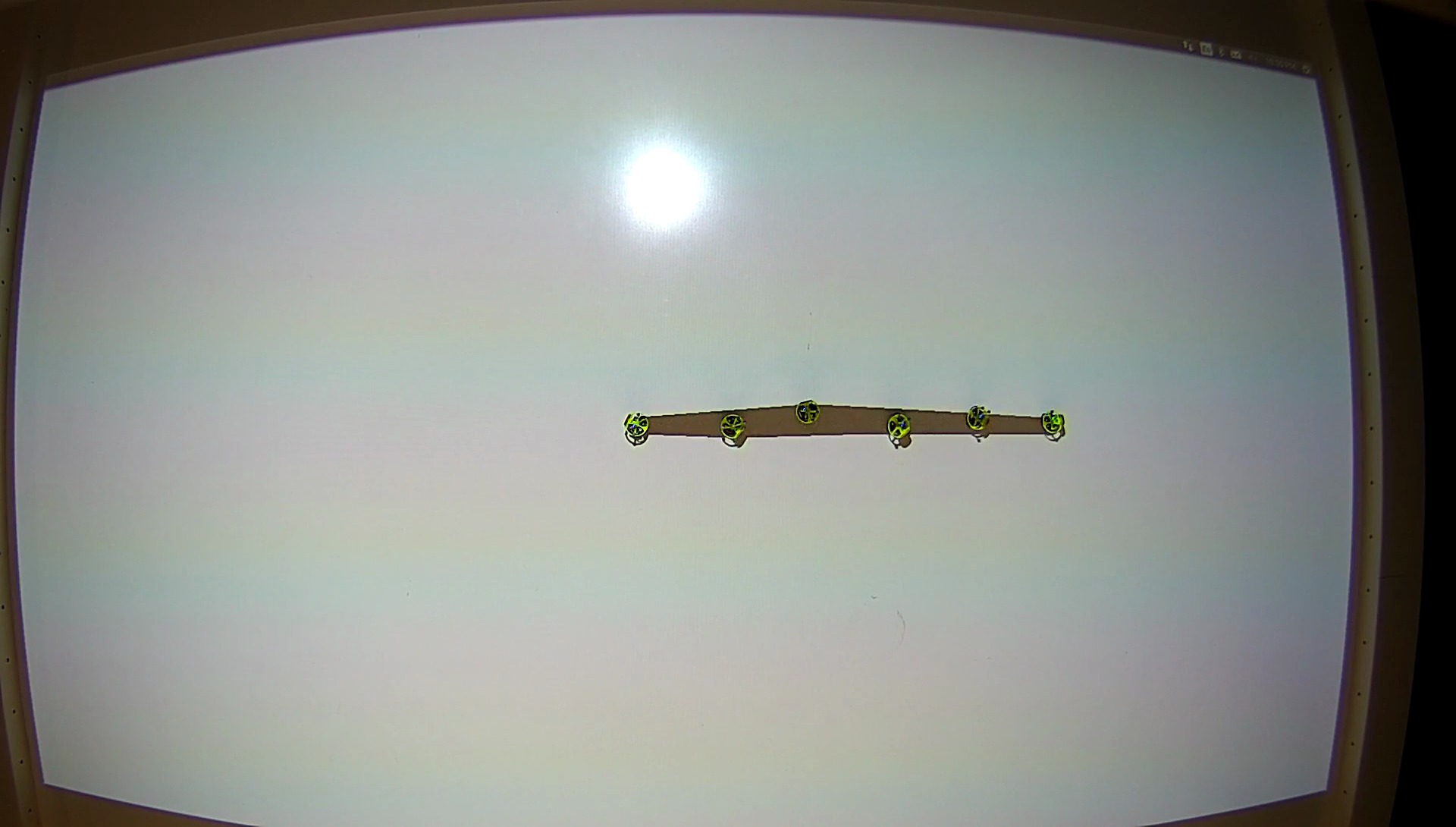}}\hfill
	\subfloat[][]{\label{subfig:formationb}\includegraphics[trim={7cm 2cm 14cm 4cm}, clip,width=.24\textwidth]{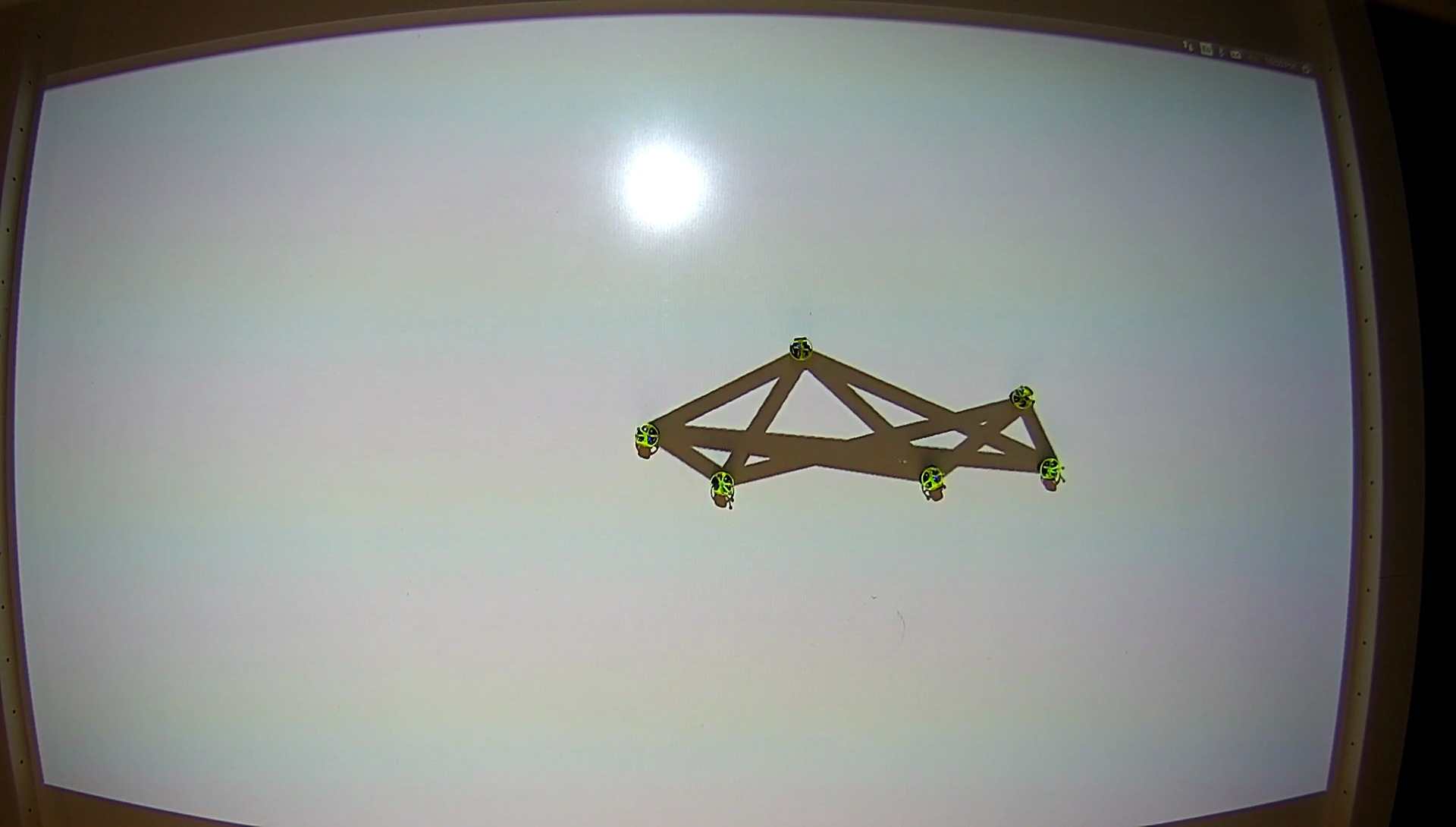}}\hfill
	\subfloat[][]{\label{subfig:formationc}\includegraphics[trim={7cm 2cm 14cm 4cm}, clip,width=.24\textwidth]{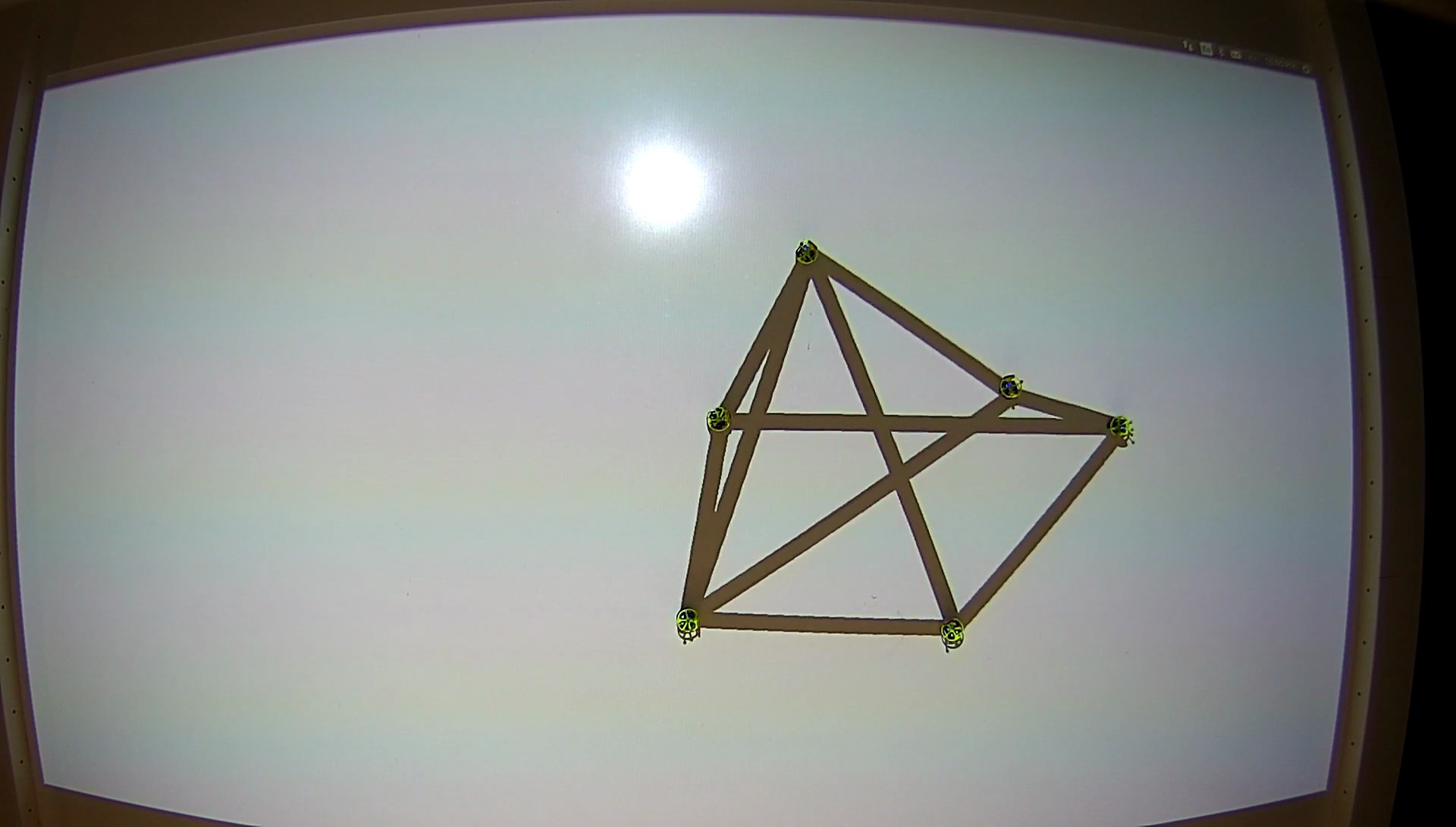}}\hfill
	\subfloat[][]{\label{subfig:formationd}\includegraphics[trim={7cm 2cm 14cm 4cm}, clip,width=.24\textwidth]{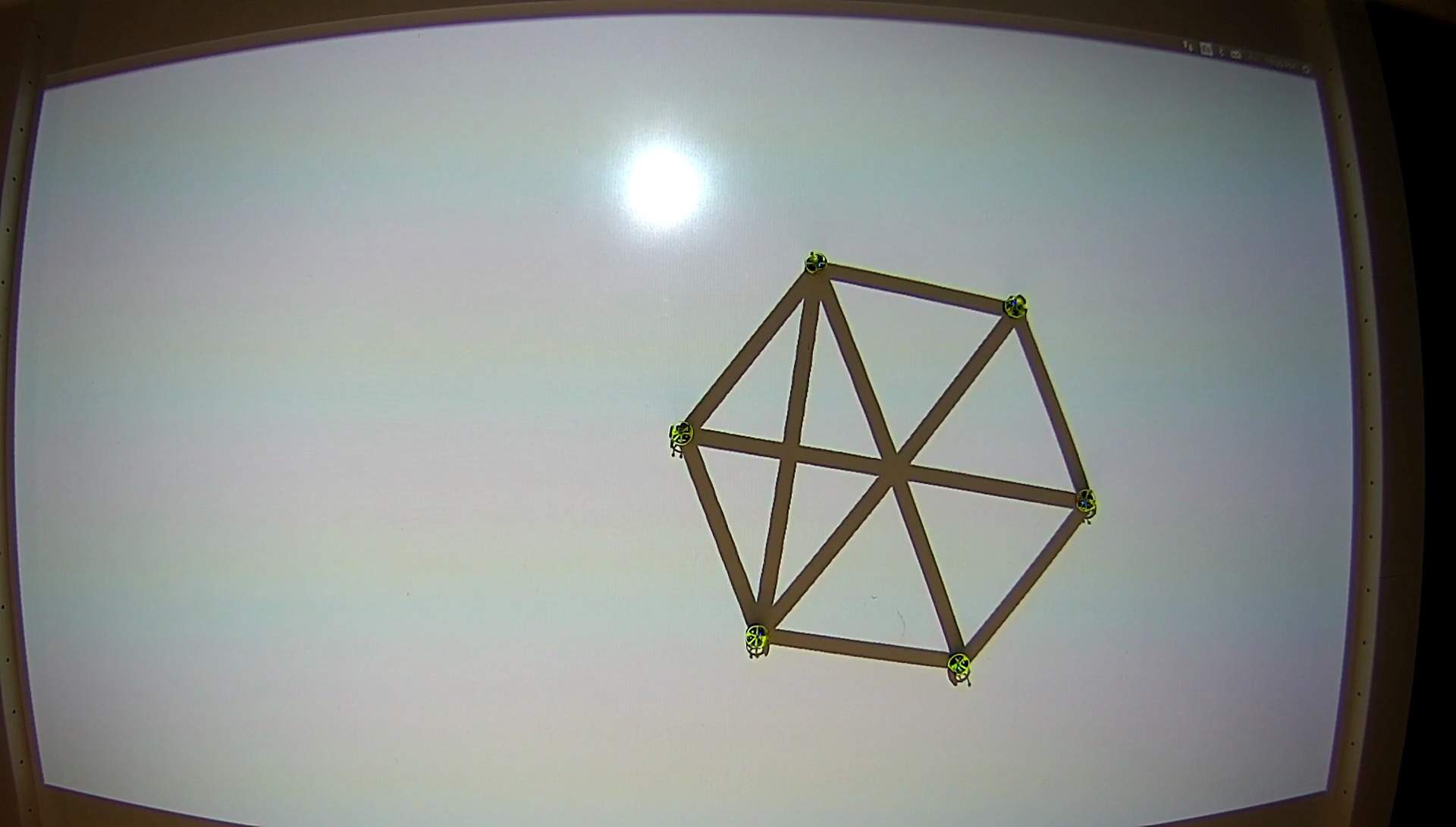}}\hfill
	\caption{A team of six small-scale differential-drive robots on the Robotarium executes formation control using \eqref{eq:formationcost} and \eqref{eq:gradformation} in the optimization program \eqref{eq:optuJ}. The edges encoding maintained distances between robots are projected onto the testbed.}
	\label{fig:formation}
\end{figure*}
\begin{figure*}
	\centering
	\hfill\subfloat[][]{\label{subfig:coveragea}\includegraphics[trim={7cm 2cm 14cm 4cm}, clip,width=.24\textwidth]{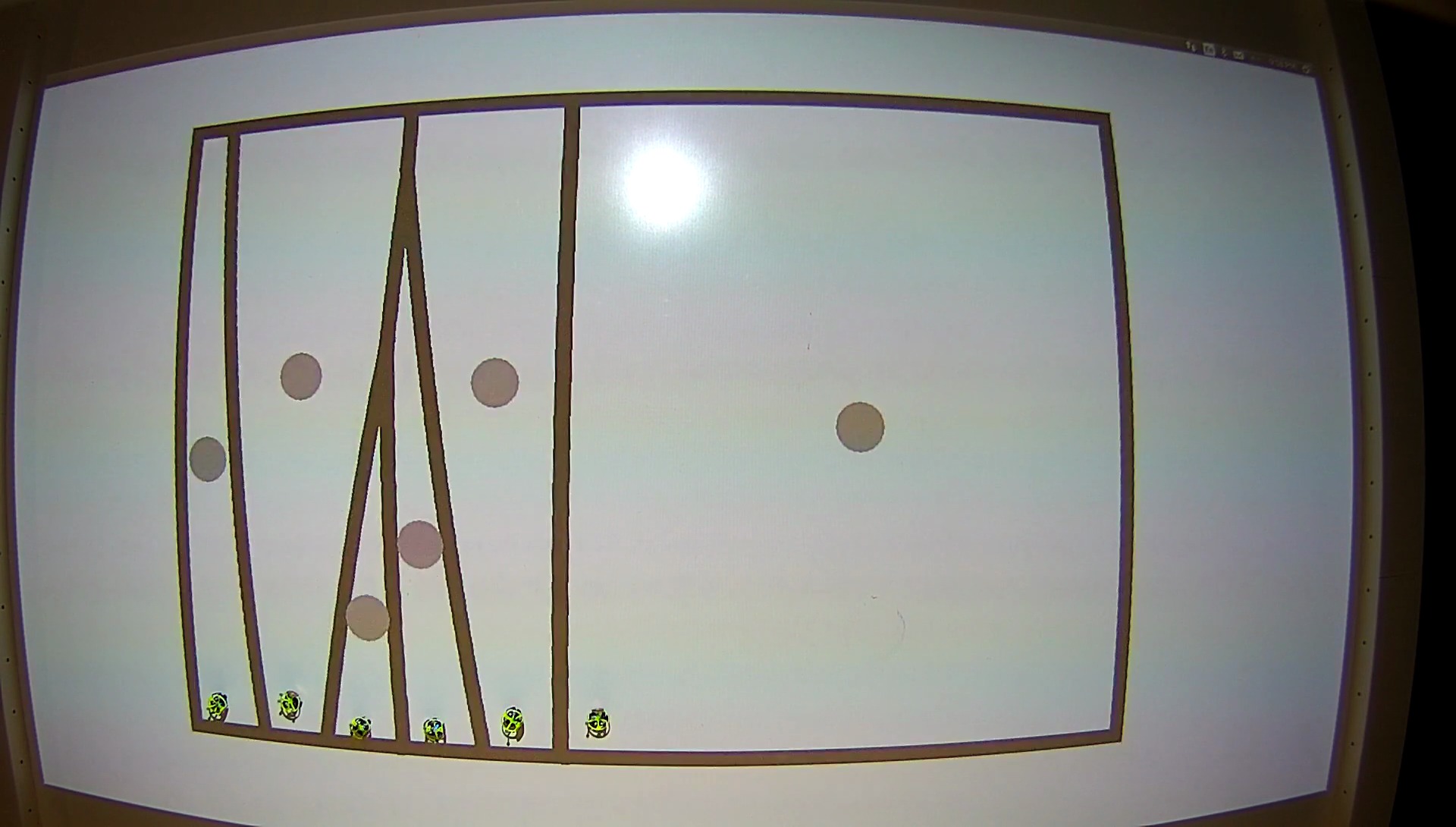}}\hfill
	\subfloat[][]{\label{subfig:coverageb}\includegraphics[trim={7cm 2cm 14cm 4cm}, clip,width=.24\textwidth]{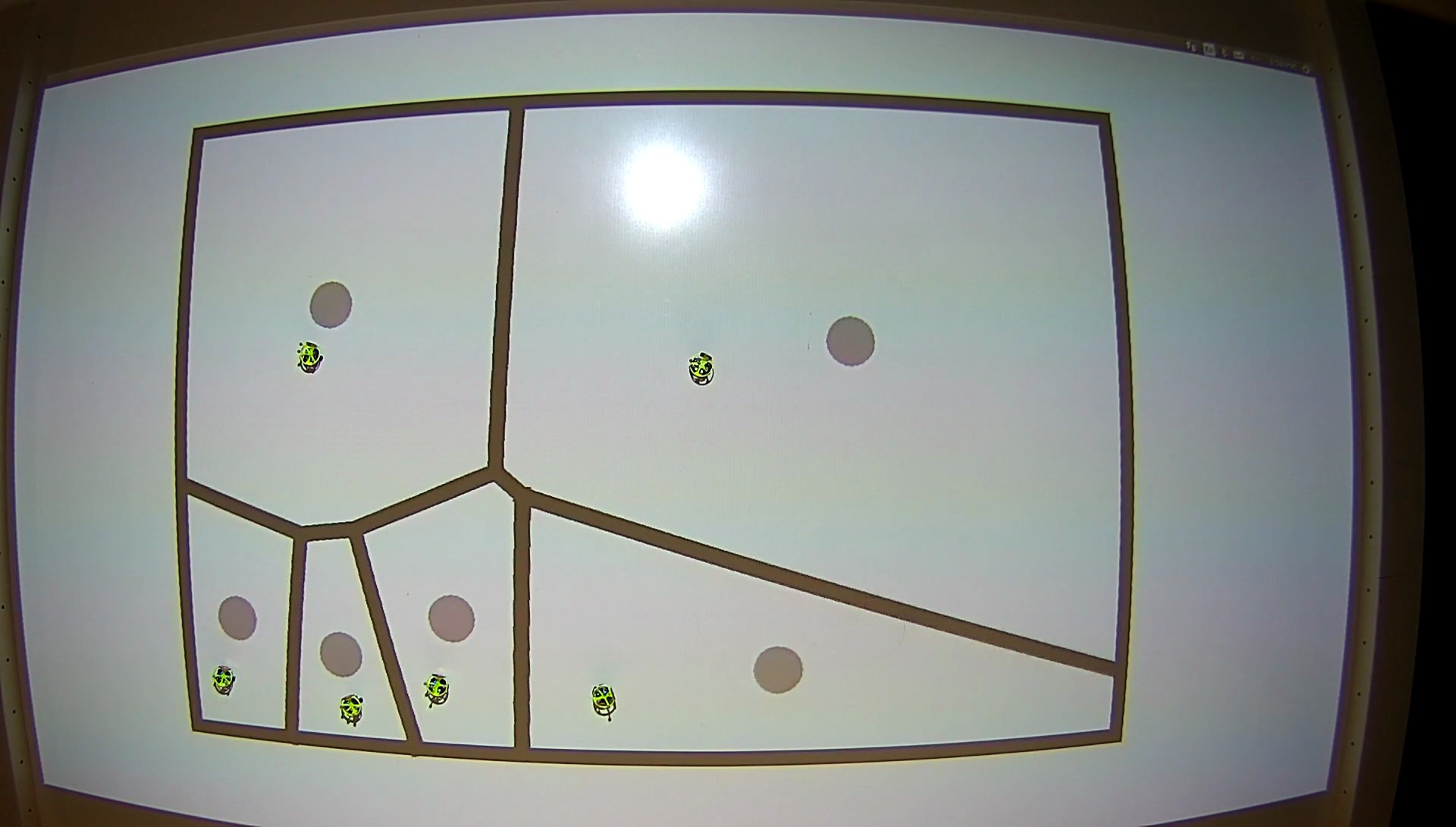}}\hfill
	\subfloat[][]{\label{subfig:coveragec}\includegraphics[trim={7cm 2cm 14cm 4cm}, clip,width=.24\textwidth]{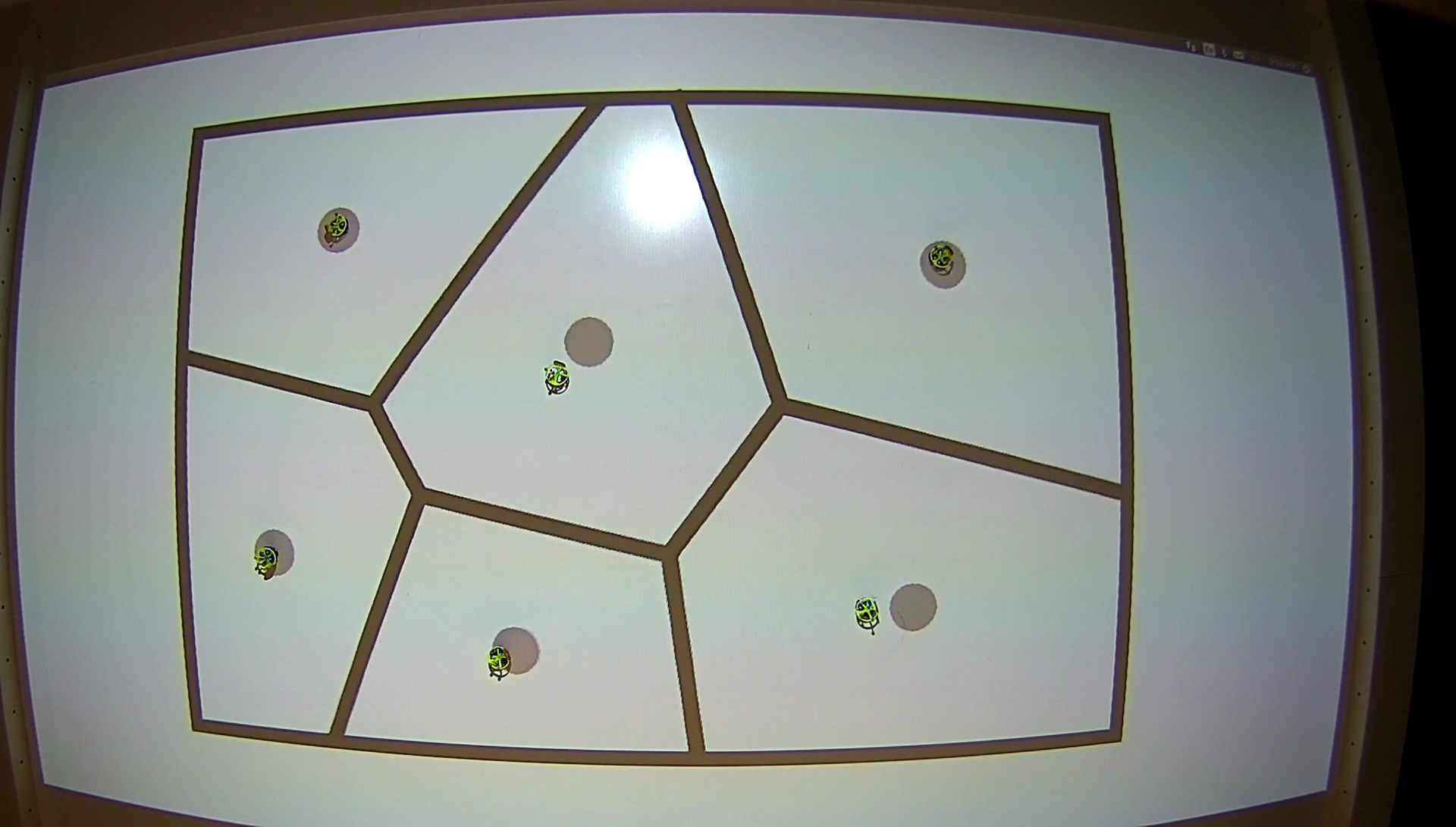}}\hfill
	\subfloat[][]{\label{subfig:coveraged}\includegraphics[trim={7cm 2cm 14cm 4cm}, clip,width=.24\textwidth]{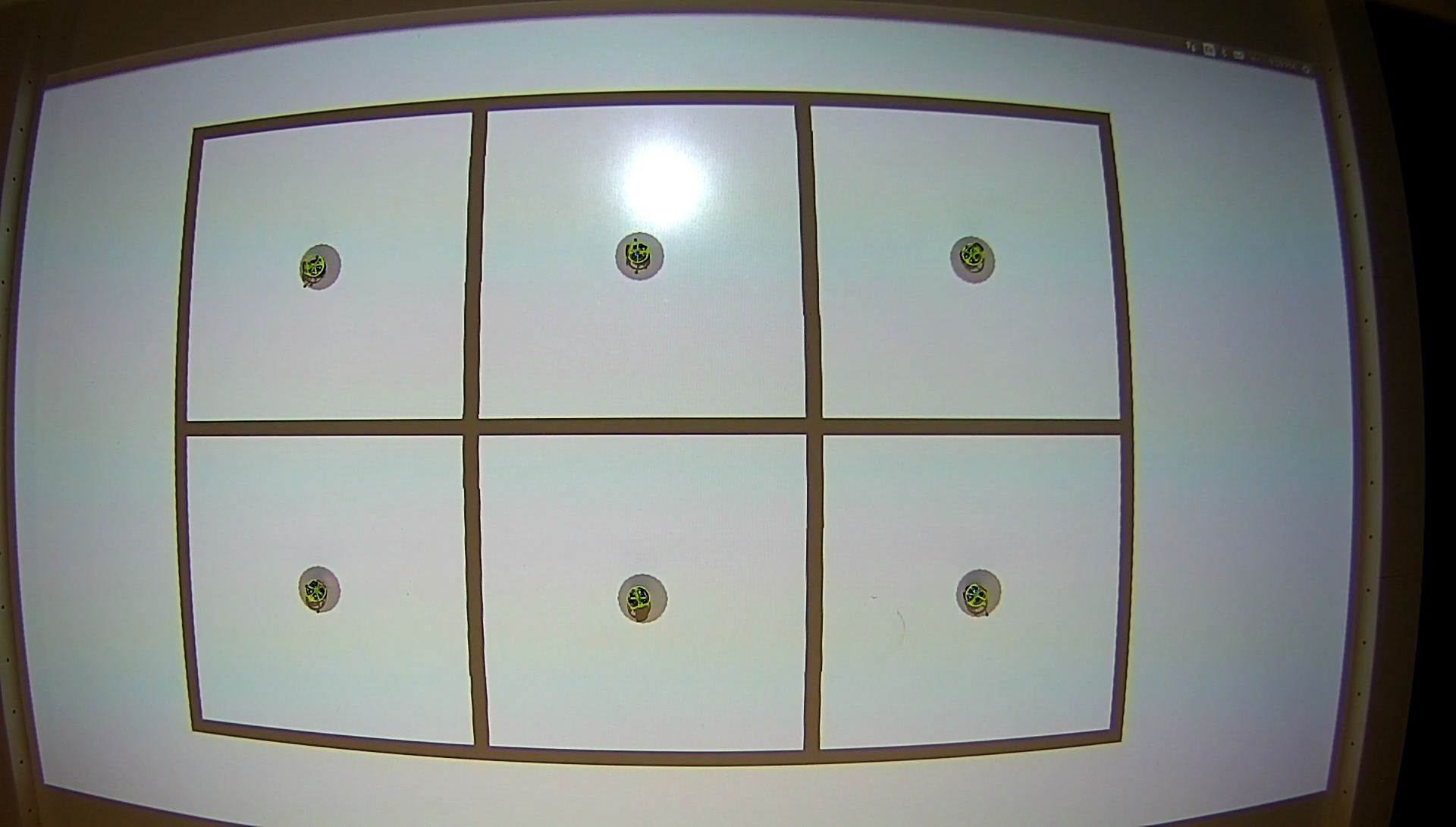}}\hfill
	\caption{A team of six small-scale differential-drive robots performs coverage control of a rectangular area on the Robotarium using \eqref{eq:loccost} and \eqref{eq:gradientcoverage} in \eqref{eq:optuJ}. The Voronoi cells of the robots are projected onto the testbed, together with their centroids, depicted as gray circles.}
	\label{fig:coverage}
\end{figure*}

The optimization problem \eqref{eq:optuJ} has been implemented with the specific expressions of $J_i$ and $\frac{\de J_i}{\de x_i}$ given in \eqref{eq:formationcost} and \eqref{eq:gradformation} in order to achieve formation control, as explained in Section~\ref{subsec:formation}. A sequence of snapshots recorded during the experiments in the Robotarium is shown in Fig.~\ref{fig:formation}: six robots are asked to assemble a hexagon specified through the inter-agent distances $d_{ij}$ in \eqref{eq:formationcost}. The edges corresponding to distances that are maintained are projected down onto the testbed and depicted as black lines in Figures~\ref{subfig:formationa}~to~\ref{subfig:formationd}.

Similarly, coverage control has been implemented using the constraint-based optimization \eqref{eq:optuJ} and the expressions of $J_i$ and $\frac{\de J_i}{\de x_i}$ in \eqref{eq:loccost} and \eqref{eq:gradientcoverage}. The results are shown in Fig.~\ref{fig:coverage}. Six robots are asked to spread over a rectangular domain. The Voronoi partition of the domain is projected on the testbed. As a result of the optimization program, the robots are moving towards the centroids of their respective Voronoi cells, represented as gray circles in Figures~\ref{subfig:coveragea}~to~\ref{subfig:coveraged}.

\subsection{Combining and Prioritizing Tasks}

In this section, we present the application of the proposed constraint-driven coordinated control to long-term environmental monitoring.

The setup of the experiment is as follows. Six robots are asked to monitor an area by performing coverage control. While executing this task, the robots must not run out of energy and must not collide with two dynamic obstacles, embodied by two additional robots moving in the environment. In order to do so, we define constraints that allow the robots to always keep enough residual energy in their batteries and to be always a minimum distance apart from the obstacles.

To accomplish the first goal, we use a method similar to the one developed in \cite{notomista2018persistification}. Assuming that the domain is endowed with charging stations, i.\,e. locations where the robots can recharge their batteries, let us define the following barrier function:
\begin{equation}\label{eq:energybarrier}
h_{e,i}(x_i,E_i) = E_i-E_{min}-k(\|x_{c,i} - x_i\|-d_{chg})^2,
\end{equation}
where $x_i$ is the position of robot $i$, $E_i$ is the energy in its battery, $E_{min}$ is the minimum residual energy we want the robots to keep, $x_{c,i}$ is the location of the charging station dedicated to robot $i$, $d_{chg}$ is the minimum distance from the charging station at which the robots can recharge their batteries (typical behavior of wireless charging technologies), and $k$ is a constant such that $k(\|x_{c,i} - x_i\|-d_{chg})^2$ upper-bounds the energy required to reach a charging station. We refer to \cite{notomista2018persistification} for a rigorous analysis.

As far as obstacle avoidance is concerned, we define, for each obstacle, the following barrier function, which ensures collision-free operations in multi-robot systems \cite{wang2017safety}:
\begin{equation}
h_{o,i}(x_i) = \|x_i-x_o\|^2 - d_o^2,
\end{equation}
where $x_o$ is the position of the obstacle and $d_o$ is the minimum distance we want the robots to maintain from the obstacle.

Combining the energy constraint $\dot h_{e,i}\ge h_{e,i}(x_i,E_i)$ and the obstacle constraint $\dot h_{o,i}\ge h_{o,i}(x_i)$ together with the coverage task constraints, the following optimization problem can be formulated:
\begin{equation}\label{eq:persistence}
\begin{aligned}
\min_{u_i,\delta_i}~&\|u_i\|^2 + \delta_i^2\\
\st~&-\frac{\de J_i}{\de x_i} u_i \ge -\alpha(-J_i(x)) - \delta_i\\
&\hspace{0.8cm}\dot h_{e,i}\ge h_{e,i}(x_i,E_i)\\
&\hspace{0.8cm}\dot h_{o,i}\ge h_{o,i}(x_i).
\end{aligned}
\end{equation}
The variable $\delta_i$ in the coverage constraint acts as a relaxation parameter, which allows the constraints related to energy and collisions to be fulfilled. This translates to \textit{trading task execution for survivability}. Consequently, this formulation allows tasks prioritization obtained by combining hard and soft constraints.

\begin{figure}
	\centering
	\subfloat[][]{\label{subfig:a}\includegraphics[trim={7cm 2cm 14cm 4cm}, clip,width=.35\textwidth]{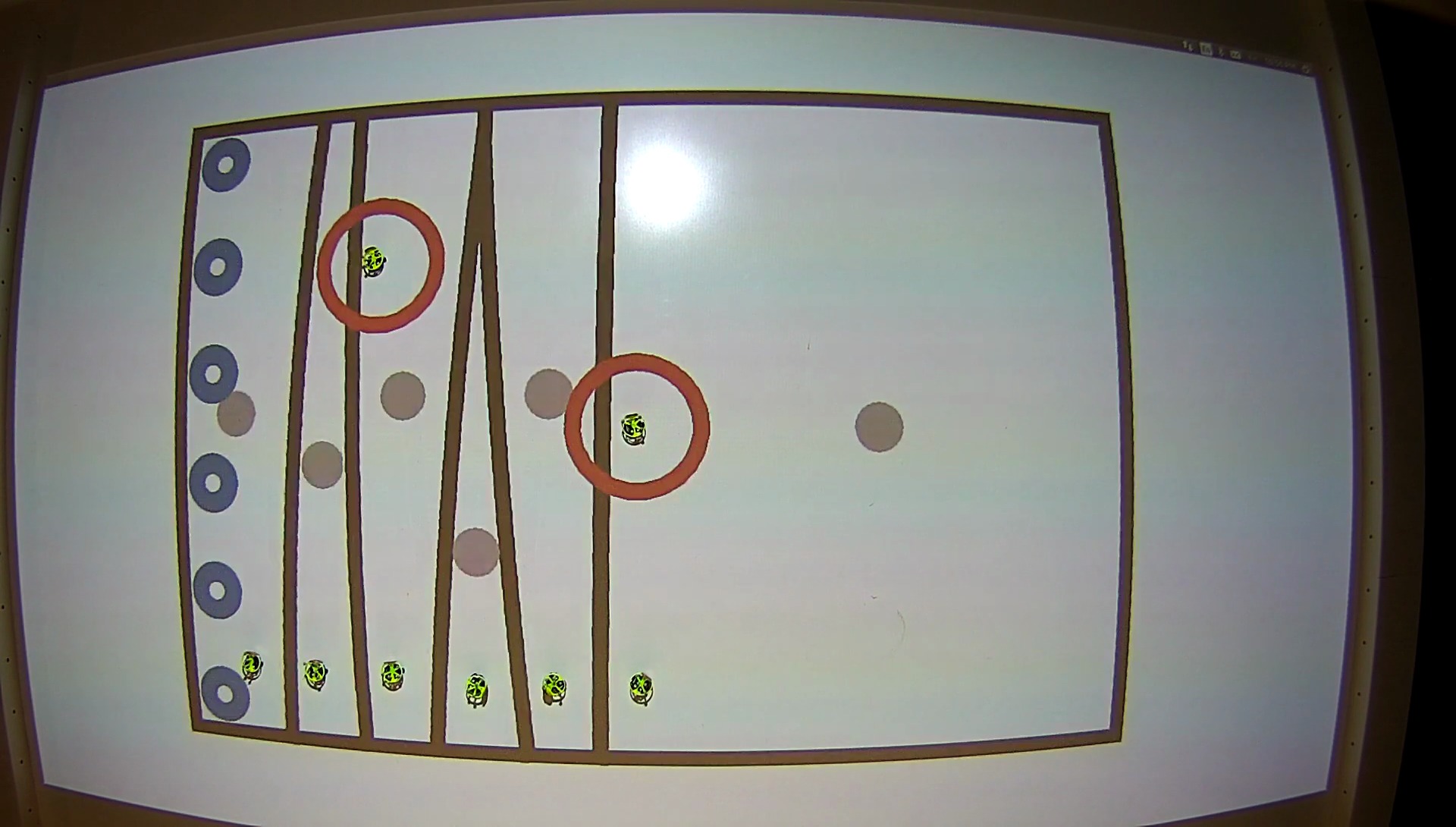}}\\
	\subfloat[][]{\label{subfig:b}\includegraphics[trim={7cm 2cm 14cm 4cm}, clip,width=.35\textwidth]{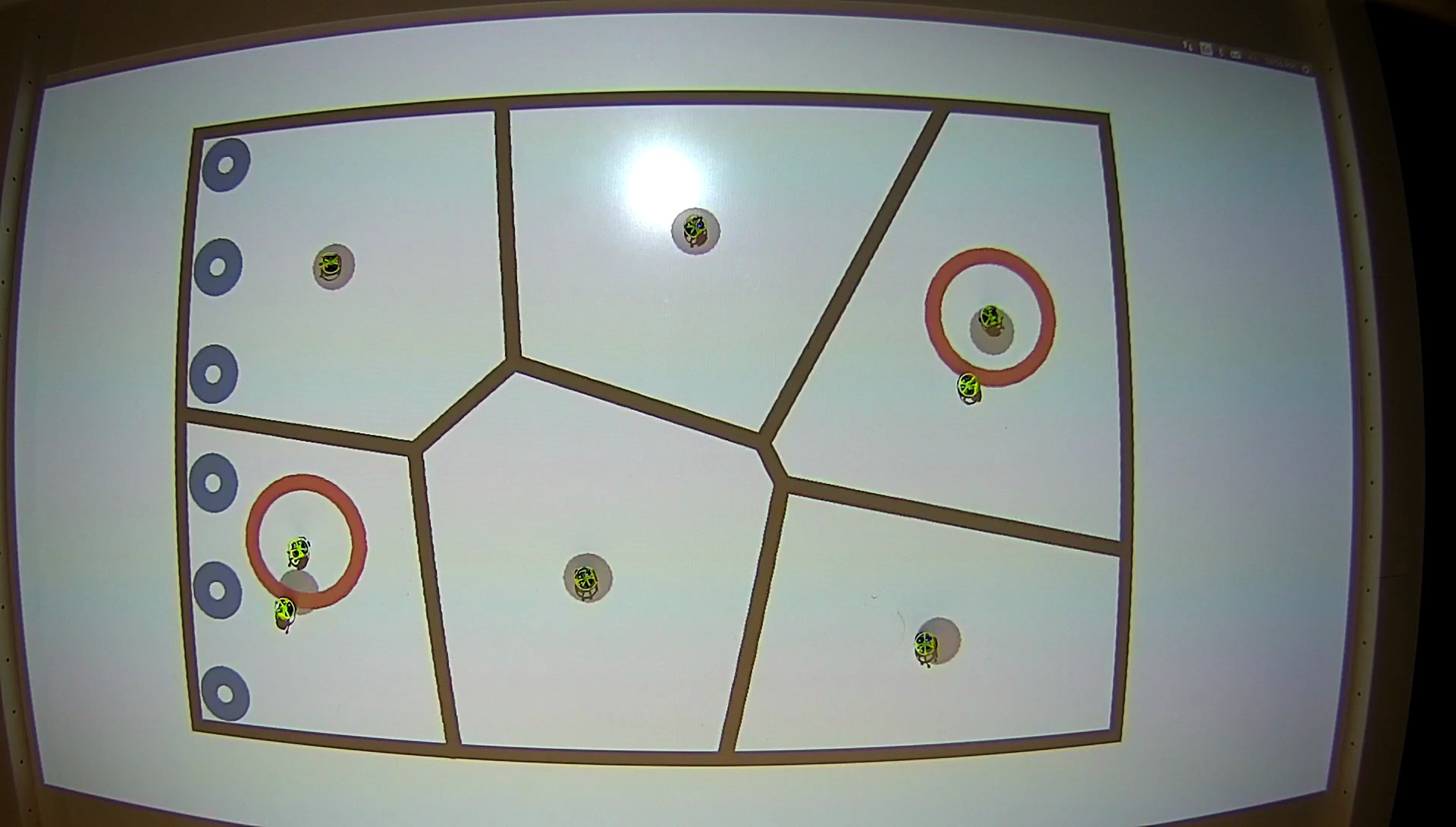}}\\
	\subfloat[][]{\label{subfig:c}\includegraphics[trim={7cm 2cm 14cm 4cm}, clip,width=.35\textwidth]{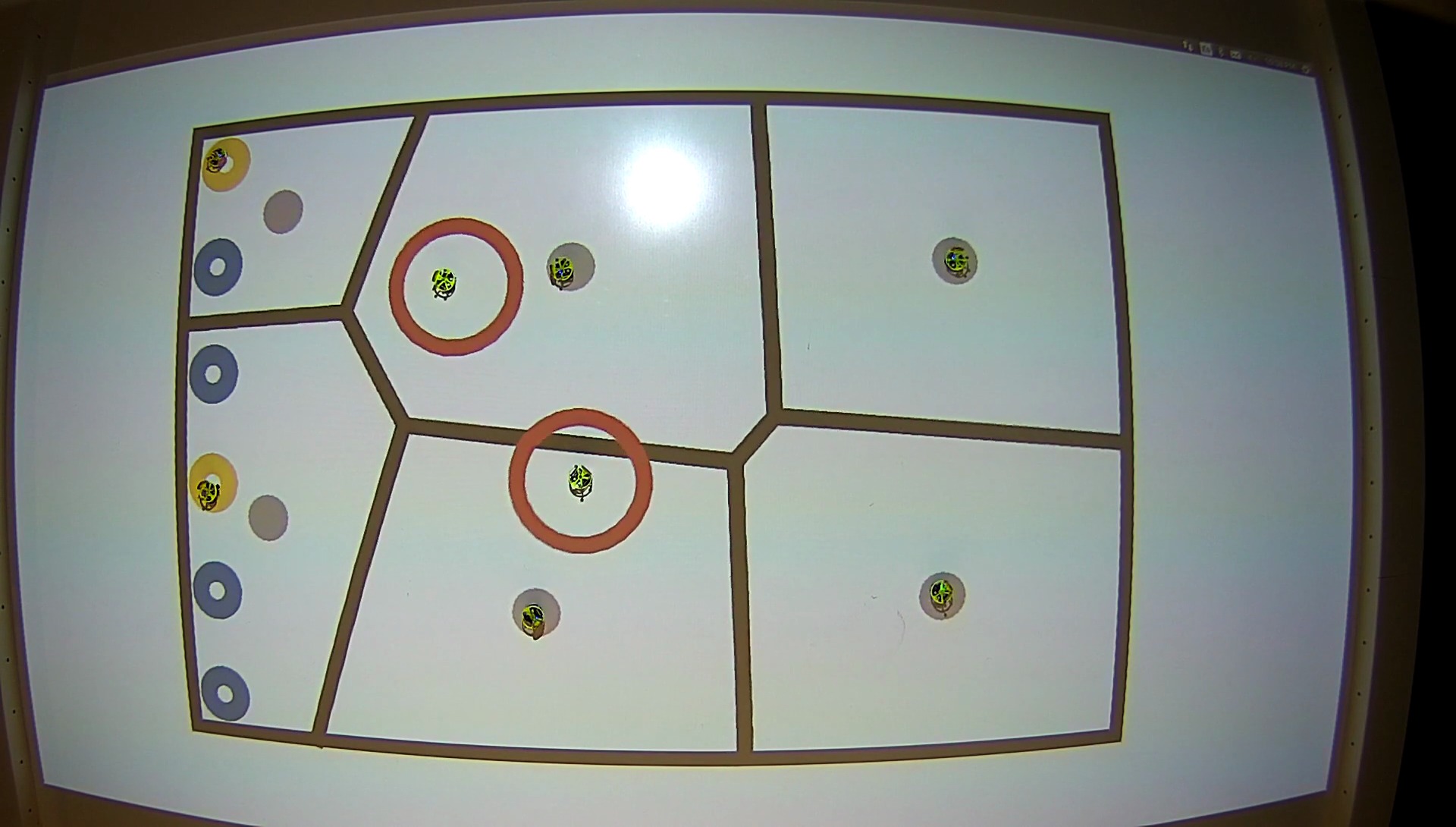}}\\
	\subfloat[][]{\label{subfig:d}\includegraphics[trim={7cm 2cm 14cm 4cm}, clip,width=.35\textwidth]{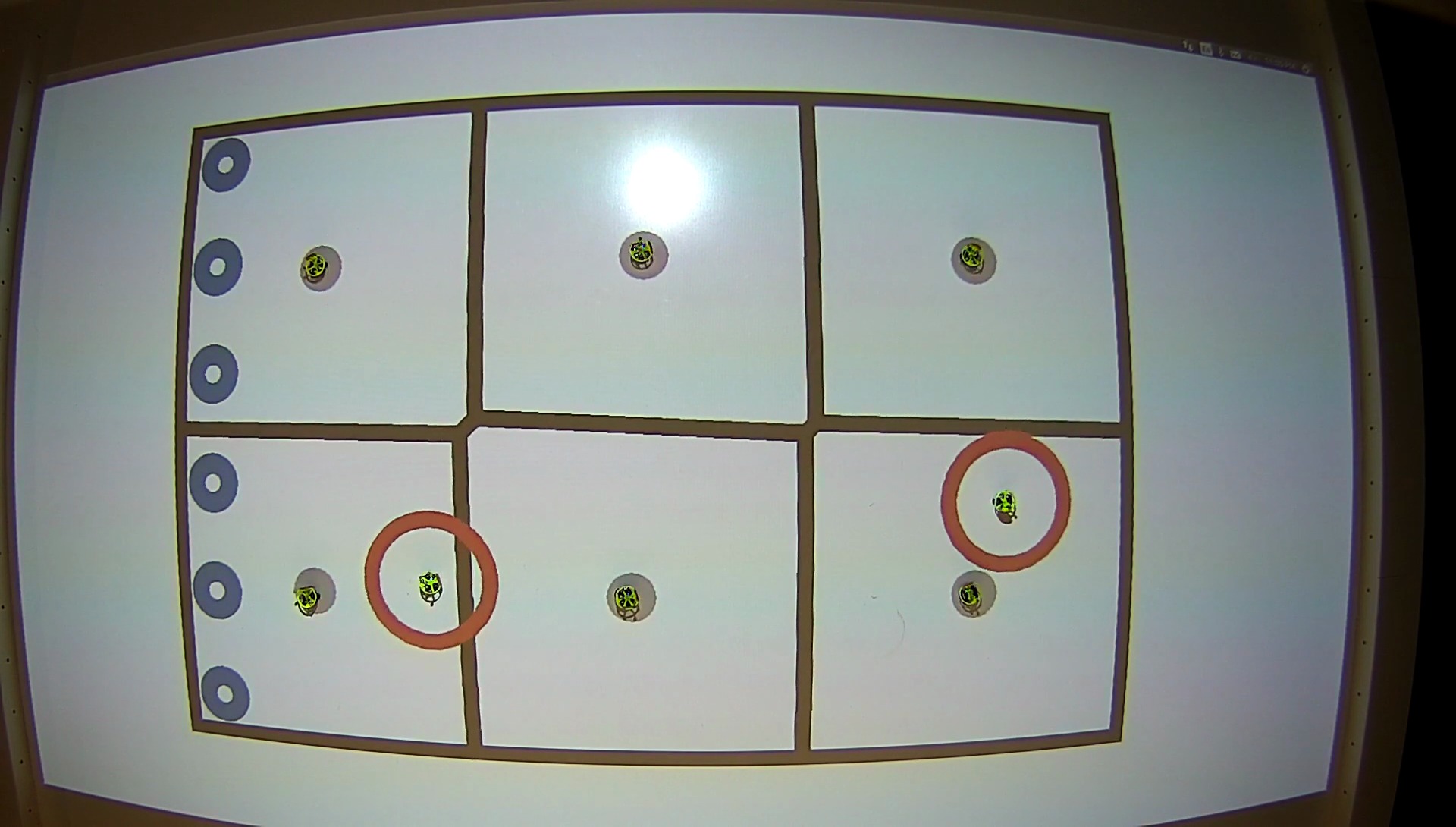}}
	\caption{A team of six robots is tasked with monitoring a rectangular domain on the Robotarium, by performing coverage control as in Fig.~\ref{fig:coverage}. This time, however, the robots are asked to perform this task over a time horizon which is much longer than their (simulated) battery life. Additionally, two more robots (circled in red) act as obstacles which have to be avoided by the remaining six robots. These execute \eqref{eq:persistence} to avoid the obstacles, go and recharge their batteries at the dedicated charging stations (blue circles on the left of the figures that turn yellow when the robots are charging), while always covering the given domain. A video of the experiments is available online \cite{youtubevideo}.}
	\label{fig:prioritizing}
\end{figure}

The results of the long-term environmental monitoring experiment are shown in Fig.~\ref{fig:prioritizing}. The Voronoi partition generated by the robots is projected down onto the testbed, as in Fig.~\ref{fig:coverage}. Arranged vertically along the left edge of the domain, there are six charging stations depicted as blue circles that turn yellow when the robots are charging. The two robots circle in red are moving in the environment acting solely as obstacles. The sequence of snapshots shows the robots starting to perform coverage (Fig.~\ref{subfig:a}), two robots avoiding a obstacles (top right and bottom left in Fig.~\ref{subfig:b}), and two robots recharging their batteries (Fig.~\ref{subfig:c}). In Fig.~\ref{subfig:d} the robots have reached a configuration corresponding to a local minimum of the locational cost \eqref{eq:loccost}. A video of the experiments can be found online \cite{youtubevideo}.

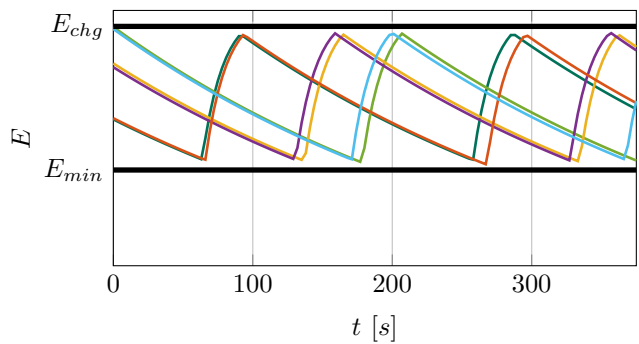
\begin{figure}
	\centering
	\begin{tikzpicture}
%		\Large % axis font size
		\begin{axis}
		[
		no marks, % remove marks
		xlabel={$t~[s]$}, % xlabel,
%		xtick={0,0.3,0.5,0.7854,1.0}, % xtick
%		xticklabels={0,30\%,$\frac{1}{2}$,$\frac{\pi}{4}$,$-e^{-i\pi}$}, % xticklabel
		ylabel={$E$}, % ylabel
		ytick={0,0.5,0.95}, % ytick
		yticklabels={0,$E_{min}$,$E_{chg}$}, % yticklabel
		%		xmin=-5, % xlim
		%		xmax=5, % xlim
		ymin=.2, % ylim
		ymax=1, % ylim
		enlarge x limits=-1, % x axis tight
		enlarge y limits=-1, % y axis tight
		%		axis equal image, % axis equal
		grid=both, % grid on
		%		xmajorgrids, % x grid on
		%		ymajorgrids, % y grid on
%		minor tick num=2, % grid minor
%		legend entries={one-column plot,test12,test34,test56}, % legend
%		legend style={nodes=right}, % legend style
%		legend pos= north west, % legend position
		width=0.48\textwidth, % image width
		height=0.28\textwidth % image height
		]
		\addplot [line width=1pt, color=myblue] table [x expr=\coordindex*3,y=x1]{data/energy.txt};
		\addplot [line width=1pt, color=mygreen] table [x expr=\coordindex*3,y=x2]{data/energy.txt};
		\addplot [line width=1pt, color=myyellow] table [x expr=\coordindex*3,y=x3]{data/energy.txt};
		\addplot [line width=1pt, color=mypurple] table [x expr=\coordindex*3,y=x4]{data/energy.txt};
		\addplot [line width=1pt, color=myorange] table [x expr=\coordindex*3,y=x5]{data/energy.txt};
		\addplot [line width=1pt, color=mycyan] table [x expr=\coordindex*3,y=x6]{data/energy.txt};
		\addplot [line width=2pt, color=black] table [x expr=\coordindex*3,y=emin]{data/energy_limits.txt};
		\addplot [line width=2pt, color=black] table [x expr=\coordindex*3,y=echarged]{data/energy_limits.txt};
		\end{axis}
	\end{tikzpicture}
	\caption{Simulated energy levels of the robots tasked with performing persistent coverage (Fig.~\ref{fig:prioritizing}). The residual energy is kept above a minimum desired value using \eqref{eq:energybarrier}. With the simulated energy dynamics, each robot experiences two charging cycles during the course of the experiment.}
	\label{fig:energylevels}
\end{figure}

Due to the limited amount of time that each experiment submitted to the Robotarium is allowed to last, we simulate the battery dynamics in such a way that the robots experience multiple charging cycles during the course of a single experiment. Fig.~\ref{fig:energylevels} shows the energy levels of the robots employed to perform coverage. The minimum desired energy level, $E_{min}$, and the value corresponding to fully charged battery, $E_{chg}$, are depicted as black thick lines. Enforcing the energy constraints using \eqref{eq:energybarrier} allows the robots to keep their energy level always above $E_{min}$.

We have shown how the constraint-driven control formulation can be used to build a minimum-energy optimization problem, whose constraints encode both the task that the robots are asked to perform and the survivability specifications, thus enabling the robust deployment of robots for long-term applications.

\section{CONCLUSIONS}
In this paper we presented a reformulation of optimization-based multi-robot tasks in terms of constrained optimization. Identifying a multi-robot task with a cost function that needs to be minimized, we leverage control barrier functions to synthesize decentralized optimization-based controllers that achieve the desired goal. The advantages of this approach include its flexibility of encoding several multi-robot tasks and the ease of combining them with different types of constraints. We showed how this flexibility can be used to enforce robot survivability and achieve long-term robot autonomy, where robustness and resilience are indispensable properties that robots have to possess. A systematic way of formulating the optimization problems for each agent of a robotic swarm is derived. Its effectiveness is demonstrated through a series of experiments using a team of ground mobile robots, culminating in a long-term environmental monitoring application.


%%%%%%%%%%%%%%%%%%%%%%%%%%%%%%%%%%%%%%%%%%%%%%%%%%%%%%%%%%%%%%%%%%%%%%%%%%%%%%%%

%%%%%%%%%%%%%%%%%%%%%%%%%%%%%%%%%%%%%%%%%%%%%%%%%%%%%%%%%%%%%%%%%%%%%%%%%%%%%%%%

%%%%%%%%%%%%%%%%%%%%%%%%%%%%%%%%%%%%%%%%%%%%%%%%%%%%%%%%%%%%%%%%%%%%%%%%%%%%%%%%
\section*{APPENDIX}

The formulation presented in this paper also allows an exact decentralized implementation of the coverage control with time-varying density functions. In \cite{lee2015multirobot}, the authors show that the control law
\begin{equation}\label{eq:utimevaringphi}
u = \left( I-\frac{\de G}{\de x} \right)\inv\left( (G(x,t)-x) + \frac{\de G}{\de t} \right)
\end{equation}
minimizes the locational cost
\begin{equation}\label{eq:loccostcortes}
\mc H(x,t) = \sum_{i=1}^{N} \int_{V_i} \|q-x_i\|^2 \phi(q,t) dq.
\end{equation}
$\phi : (q,t)\in D\times\R_+ \mapsto \phi(q,t)\in\R_+ $ is a time-varying density function, which specifies the importance of point $q$ at time $t$. The cost in \eqref{eq:loccostcortes} is equivalent to the one defined in \eqref{eq:loccost} when the centroids $G_i(x)$ are calculated weighting the points in the domain according to the value of the density function associated to them, as shown in \cite{cortes2004coverage}.

However, inverting the matrix $I-\frac{\de G}{\de x}$ in \eqref{eq:utimevaringphi} cannot be done in a decentralized fashion. For this reason, in \cite{lee2015multirobot}, the inverse is approximated by a truncated Neumann series as
\[
\left(I-\frac{\de G}{\de x}\right)\inv \approx I+\frac{\de G}{\de x},
\]
which, on the contrary, can be evaluated only based on information about neighboring robots.

With the formulation presented in this paper, instead, by implementing the optimization problem \eqref{eq:optuJ} in Proposition~\ref{prop:decentralizedtask}, each robot has to solve
\[
\begin{aligned}
\min_{u_i,\delta_i}~&\|u_i\|^2+|\delta_i|^2\\
\st~&-(x_i-G_i(x,t))\tr\left( I-\frac{\de G_i(x,t)}{\de x_i} \right) u_i\\
&\ge -\alpha(-J_i(x,t))-(x_i-G_i(x,t))\tr\frac{\de G_i(x,t)}{\de t}-\delta_i,
\end{aligned}
\]
which is both exact and decentralized.

\section*{ACKNOWLEDGMENT}

The authors would like to thank Professor Jonathan N. Pauli for helpful discussions about ecology.

%%%%%%%%%%%%%%%%%%%%%%%%%%%%%%%%%%%%%%%%%%%%%%%%%%%%%%%%%%%%%%%%%%%%%%%%%%%%%%%%

\bibliographystyle{IEEEtran}
\bibliography{bib/IEEEabrv,bib/acc2019tasksasconstraints}

\end{document}